\documentclass[letterpaper]{article} 
\usepackage[preprint]{aaai2027}  
\usepackage[hyphens]{url}  
\usepackage{graphicx} 
\usepackage{natbib}  
\usepackage{caption} 
\usepackage{algorithm}
\usepackage{algorithmic}

\usepackage{newfloat}
\usepackage{listings}
\DeclareCaptionStyle{ruled}{labelfont=normalfont,labelsep=colon,strut=off} 
\floatstyle{ruled}
\newfloat{listing}{tb}{lst}{}
\floatname{listing}{Listing}

\usepackage{booktabs}

\usepackage{amsmath}
\usepackage{amssymb}
\usepackage{amsthm}
\usepackage{dsfont}      

\DeclareMathOperator*{\argmax}{arg\,max}  

\def\eg{{\it e.g.}}

\def\ie{{\it i.e.}}

\newtheorem{theorem}{Theorem}

\newtheorem{lemma}{Lemma}
\newtheorem{corollary}{Corollary}

\title{Orienteering Problem with Uncertain Time-Varying Rewards: \\
Framework and Benchmark for Everyday Service Robotics}
\author{
    Masafumi Endo\textsuperscript{\rm 1}\corresponding,
    Kohei Honda\textsuperscript{\rm 1,\rm 2},
    Yuu Jinnai\textsuperscript{\rm 1},
    Ryo Yonetani\textsuperscript{\rm 1}
}
\affiliations{
    \textsuperscript{\rm 1}CyberAgent Inc.\\
    \textsuperscript{\rm 2}Nagoya University\\
    \{endo\_masafumi, honda\_kohei, jinnai\_yu, yonetani\_ryo\}@cyberagent.co.jp
}

\begin{document}

\maketitle

\begin{abstract}
We present the orienteering problem with uncertain time-varying rewards (OP-UTVR), a novel variant of the orienteering problem (OP).
While most existing OP formulations assume rewards to be known in advance, practical applications involve uncertain and time-varying rewards, as with shifting customer demand for delivery agents.
OP-UTVR relaxes this assumption by allowing agents to estimate reward dynamics from observations and forecast future rewards.
This enables informed routing decisions despite stochastic reward changes and inevitable prediction errors.
We address this problem using three planners that differ in planning horizon and online adaptivity, and derive theoretical bounds on their performance under reward stochasticity.
We further introduce a mobile service robot benchmark for OP-UTVR, where a robot navigates among pedestrians in indoor environments.
Experiments reveal trade-offs between planning horizon and adaptivity, and demonstrate the effectiveness of long-horizon planning with online adaptation.
\end{abstract}


\section{Introduction}
\label{sec:introduction}
Rewards in the real world are often dynamic and inherently uncertain. Imagine a taxi driver targeting a high-traffic area. 
Passenger demand shifts with commuting patterns, yet individual trajectories of potential riders remain stochastic and unknown. 
Handling uncertain reward dynamics matters even more in disaster response; it is crucial to locate those stranded within affected zones, where the probability of successful rescue diminishes over time due to the highly uncertain dynamics of the environment.

We are interested in enabling mobile agents that can plan and execute an effective path to collect such uncertain time-varying rewards across multiple destinations. 
A natural starting point for these selective routing problems is the orienteering problem (OP; \citeauthor{tsiligirides1984heuristic}, \citeyear{tsiligirides1984heuristic}), where the objective is to find a sequence of locations to visit that maximizes accumulated rewards under resource constraints. 
While the OP has inspired a variety of related applications, such as persistent monitoring~\citep{yu2016correlated}, unknown terrain exploration~\citep{peltzer2022figop}, or search-and-rescue operations~\citep{jorgensen2024matroid}, existing work assumes that rewards are static or follow dynamics known to the agent~\citep{ma2017spatio,cao2024heuristic}. 
These assumptions rarely hold in practice and create a performance gap when actual rewards vary and deviate from expectations.

This gap motivates the main contribution of this work: a novel variant of the OP named \emph{orienteering problem with uncertain time-varying rewards (OP-UTVR)}.
Unlike existing OP variants with static or known time-varying rewards, OP-UTVR captures two challenges arising from reward uncertainty: 1) Rewards vary stochastically over time, like a crowd of potential taxi riders moving in an apparently random manner; and 2) agents are only allowed to `predict' such time-varying rewards for a limited time horizon, which inherently involves prediction errors relative to actual rewards. 

To solve the proposed OP-UTVR, we investigate three planning algorithms that leverage reward prediction differently, and theoretically characterize their performance.
Specifically, \emph{One-Step Planner} greedily selects the next location based on predicted immediate rewards. 
\emph{Offline Planner} predicts time-varying future rewards until the time limit, and plans an entire path that maximizes the accumulated rewards. 
Finally, \emph{Adaptive Online Planner} also plans the entire path but replans at each visited location, using updated observations to adapt to reward uncertainty. 
We prove that Adaptive Online Planner achieves near-optimal performance under accurate reward estimation, with its advantage over Offline Planner growing as reward stochasticity increases.

As an evaluation benchmark for OP-UTVR, we develop a \emph{mobile service robot benchmark} that simulates a robotic agent navigating among moving pedestrians in indoor environments reconstructed from real-world 3D scans~\citep{ramakrishnan2021hm3d,xia2018gibson}. 
Figure~\ref{fig:teaser} illustrates the motivating application of this benchmark, where the robot aims to maximize rewards given by pedestrian interactions for service activities such as greeting or advertising. 
This task naturally induces uncertain time-varying rewards: rewards change as pedestrians move, while their intentions remain unknown and future motion grows increasingly unpredictable.
Our comprehensive evaluation reveals trade-offs among the three planners between planning horizon and adaptivity to uncertain pedestrian movements. 
The results demonstrate the effectiveness of integrating long-horizon planning with adaptive replanning through runtime observations.

\begin{figure}[t]
    \centering
    \includegraphics[width=.95\linewidth]{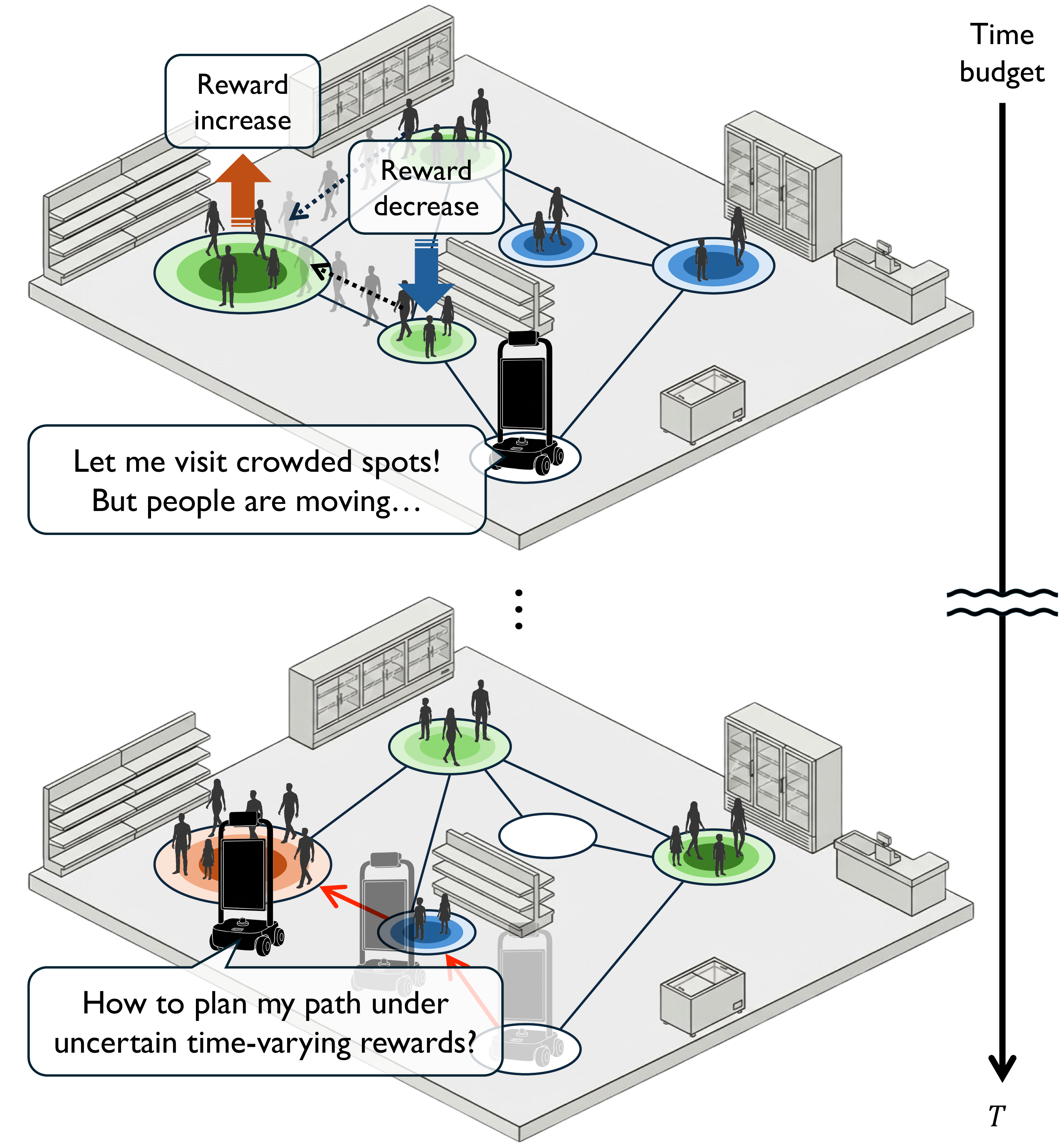}
    \caption{\textbf{Motivating application: Service robot benchmark for OP-UTVR.} 
    A robot plans a path to maximize pedestrian interactions within a time budget $T$.
    Vertices are gathering spots, with colors indicating visitor density (red, green, blue, white from high to empty).
    Rewards change as visitors move, a setting that traditional OP cannot handle.
    }
    \label{fig:teaser}
\end{figure}

\section{Preliminaries}
\label{sec:preliminaries}

The OP addresses selecting and ordering vertices to visit, to maximize collected rewards under a time budget.
We review the classical formulation and its extension to \emph{known} time-varying rewards to differentiate them from OP-UTVR. 

\subsection{Orienteering Problem with Static Rewards}
\label{op-str}

An instance of OP is defined by tuple $(\mathcal{V}, \mathcal{E}, v_0, r, d, T_\mathrm{max})$. Here, $\mathcal{V}$ is the set of vertices and $\mathcal{E}$ is the set of edges with $e(v, v') \in \mathcal{E}$ connecting vertices $v, v'\in\mathcal{V}$.
Each vertex has an associated static reward (profit) $r: \mathcal{V} \rightarrow \mathbb{R}_{\geq}$, while each edge has an associated travel time $d: \mathcal{E} \rightarrow \mathbb{N}$ (\ie, discrete time steps).
Starting from the initial vertex, $v_0\in\mathcal{V}$, the agent moves between the vertices along the edges to collect the rewards while consuming travel times. 
The objective of OP is to find a path, \ie, a sequence of distinct connected vertices, $V = (v_0, v_1,..., v_{n-1})$ that maximizes the total rewards $R_\mathrm{tot}(V)$ within the time budget $T_\mathrm{max}$ as follows:
\begin{align}
    R_\mathrm{tot}(V) &= \sum_{i=0}^{n-1} r(v_i) \label{eq:op_return}\\
    &\mathrm{s.t.}\;\sum_{i=0}^{n-2} d(e(v_i, v_{i+1})) \leq T_\mathrm{max}. \label{eq:op_constraint}
\end{align}

Since rewards are static and certain, the optimal path can be planned entirely in advance.
This assumption holds in applications such as tourist trip planning~\citep{vansteenwegen2011orienteering}, where the value of each location is predetermined.

\subsection{Extension to Known Time-Varying Rewards}
\label{sec:known-tvr}

Prior work extends the OP to handle time-varying rewards, where the reward at each vertex depends on the visit time.
Formally, we consider a reward function that maps a combination of vertex and time to a reward: $r_\mathrm{tv}: \mathcal{V} \times \mathbb{N} \rightarrow \mathbb{R}_{\geq}$.
The objective is to find a path $V = (v_0, v_1, \ldots, v_{n-1})$ with corresponding visit times $(t_0, t_1, \ldots, t_{n-1})$ that maximizes:
\begin{align}
    R_\mathrm{tot}(V) = \sum_{i=0}^{n-1} r_\mathrm{tv}(v_i, t_i).
    \label{eq:tvr_return}
\end{align}

Unlike the static case, the same vertex provides different rewards depending on when it is visited, \ie, the solution requires optimization of both vertex selection and visit timing.
\citet{ma2017spatio} introduced this extension, while \citet{cao2024heuristic} improved efficiency via heuristic search.
However, both assume that $r_\mathrm{tv}(v, t)$ is completely known in advance.
This assumption is valid only for applications where reward dynamics follow known patterns, such as factory production schedules or periodic traffic flows.
When rewards arise from uncontrollable factors, such as pedestrians whose intentions are unobservable, their dynamics are inherently stochastic, which makes exact prediction impossible.

\section{Orienteering Problem with Uncertain Time-Varying Rewards (OP-UTVR)}
\label{sec:problem_formulation}

The proposed OP-UTVR is a novel variant of the traditional OP with two key features: 1) the rewards vary over time under \emph{uncertain} dynamics that nevertheless can be predicted from observations for a short time horizon; 2) the agent can revisit the same vertices consecutively or repetitively in a solution path. A motivating example is a mobile service robot that needs to interact with as many pedestrians in the environment as possible by predicting their future locations, within its battery limit. Below, we formulate each key feature.

\subsection{Problem Formulation}
\label{sec:problem_formulation_formulation}

OP-UTVR consists of a tuple $(\mathcal{V}, \mathcal{E}, v_0, r_\mathrm{tv}, \hat{P}, P, d, T_\mathrm{max})$.
Unlike the problems in Sec.~\ref{sec:known-tvr}, we tackle a more challenging case where the \emph{true} time-varying reward $r_\mathrm{tv}$ will remain unknown until actual time $t$.

While $r_\mathrm{tv}$ is not given, we assume two properties that the agent can exploit.
First, we assume that the underlying reward dynamics $P$ is a Markov chain over vertices $\mathcal{V}$. 
That is, $r_\mathrm{tv}$ moves between vertices following a transition matrix $P$.\footnote{More complicated transition models, such as non-linear or non-Markovian dynamics, are left for future work.} 
Specifically, let $P \in (\Delta^{|\mathcal{V}|-1})^{|\mathcal{V}|}$ be the true (and unknown) row-stochastic transition matrix, where $p_{u,v}$ represents the probability of reward transitioning from vertex $u$ to vertex $v$. %
Let $\mathbf{r}_\mathrm{tv}: \mathbb{N} \rightarrow \mathbb{R}_{\geq}^{1 \times |\mathcal{V}|}$ represent the true reward value at each vertex at timestep $t$.
$\mathbf{r}_\mathrm{tv}$ evolves following the underlying transition matrix $P$, and its expectation satisfies:
\begin{equation}\label{eq:tv_reward}
    \mathbb{E}[\mathbf{r}_\mathrm{tv}(t+1)] = \mathbf{r}_\mathrm{tv}(t) \, P.
\end{equation}

The other assumption is that the agent has access to an estimate of the transition matrix, $\hat{P} \approx P$. 
$\hat{P}$ can be obtained in any way; in our benchmark, we estimate it from prior observations (Sec.~\ref{sec:simulation_experiments}).

The agent may use $\hat{P}$ to estimate $r_\mathrm{tv}$ in future using Eq.~\eqref{eq:tv_reward}. 
Let $\hat{r}_\mathrm{tv}$ denote the time-varying reward estimated by $\hat{P}$.
Then, the estimated future reward $\hat{\mathbf{r}}_\mathrm{tv}(t)$ at timestep $t$ can be computed by multiplying $\hat{P}$ for $t$ times to the observation of $\mathbf{r}_\mathrm{tv}(0)$, which is available to the agent at the beginning of the traversal.
\begin{equation}\label{eq:tv_model_t0}
    \mathbb{E}[\hat{\mathbf{r}}_\mathrm{tv}(t) \mid \mathbf{r}_\mathrm{tv}(0)] = \mathbf{r}_\mathrm{tv}(0) \, \hat{P}^t.
\end{equation}
Yet with uncertainty, this enables the agent to make its decisions with estimates of the expected future reward.

To summarize, unlike existing time-varying reward OP problems, the agent in OP-UTVR knows neither $r_\mathrm{tv}$ nor $P$ in advance; it only observes $\mathbf{r}_\mathrm{tv}(t)$ at each timestep $t$ and has access to the estimate $\hat{P}$ for reward prediction.

\subsection{Solution}
\label{sec:ot_utvr_solution}
The challenge of OP-UTVR requires the solutions to be formatted differently from prior formulations in two ways.

First, we allow the agent to make decisions \emph{online}. In OP-UTVR, an agent may benefit from adapting to unexpected outcomes during the traversal. 
While $r_\mathrm{tv}(\cdot, t)$ is unknown in advance, we assume the agent observes the rewards at all vertices at timestep $t$.
This also leads to an update to the estimate of the future reward. At timestep $0$, the estimated reward at timestep $t$ is computed as in Eq.~\eqref{eq:tv_model_t0}. However, when the agent is at timestep $t_\mathrm{now}$, it may update its estimate using the observation of $\mathbf{r}_\mathrm{tv}(t_\mathrm{now})$, which is likely to be more informative than the observation at previous timesteps:
\begin{equation}\label{eq:tv_model}
    \mathbb{E}[\hat{\mathbf{r}}_\mathrm{tv}(t) \mid \mathbf{r}_\mathrm{tv}(t_\mathrm{now})] = \mathbf{r}_\mathrm{tv}(t_\mathrm{now}) \, \hat{P}^{t-t_\mathrm{now}}.
\end{equation}
Thus, we allow the agent to choose its path adaptively during traversal in OP-UTVR. Concretely, the solution of OP-UTVR is a \emph{policy} $\pi$ that computes the next vertex to visit given the current circumstances (\eg, current agent and reward positions, remaining timesteps).

The other point to consider is that, in OP-UTVR, the agent does not have to visit a new previously unvisited vertex to obtain a new reward. 
Because the rewards are also moving, the agent may obtain rewards from visiting the same vertex more than once. 
For example, the agent may choose to visit the same vertex consecutively to wait at the current position, \eg, $V=(v, v, v, v', \dots)$ or repetitively \eg, $V=(v, v', v'', v,\dots)$. 
For the sake of simplicity, we denote $d$ to represent both the travel time and the wait time  (\ie, visiting the same vertex consecutively). 
We assume $d(e(u, u)) > 0$ as otherwise the agent may visit the same vertex infinitely many times. 

These modifications accommodate applications where rewards are not strictly one-time payments upon the first visit, but are instead accrued continually while occupying a vertex. This effectively models applications such as digital advertisements in shopping malls, where viewing duration by visitors is critical as time-varying rewards, and different visitors may arrive at the same location at different time intervals.
Our experiments in Sec.~\ref{sec:simulation_experiments} consider an agent in such scenarios.

The objective of the agent $R_\mathrm{tot}(\pi)$ is to maximize the expected total reward obtained by its policy $\pi$:
\begin{equation}\label{eq:objective}
    R_\mathrm{tot}(\pi) = \underset{V \sim \pi}{\mathbb{E}}[R_\mathrm{tot}(V)] = \underset{V \sim \pi}{\mathbb{E}}[\sum_{i=0}^{n-1} r_\mathrm{tv}(v_i,T_i)],
\end{equation}
where $T_i = \sum_{j=0}^{i-1} d(e(v_j, v_{j+1}))$ is the timestep when the agent arrives at vertex $v_i$, $V$ is a path sampled according to the decisions of $\pi$, and $n$ is the length of $V$.
The objective function (Eq.~\eqref{eq:objective}) is equivalent to Eq.~\eqref{eq:tvr_return} but with expectation over $V$, which depends on the choice of $\pi$.

Note that OP-UTVR with known dynamics can be cast as a Markov decision process (MDP; \citeauthor{puterman2014markov}, \citeyear{puterman2014markov}).
\begin{lemma}
\label{lem:mdp}
Assume $\hat{P} = P$. An instance of OP-UTVR is an instance of an MDP.
\end{lemma}
\noindent The optimal policy of this MDP can in principle be computed in advance by dynamic programming.
This requires the true transition matrix $P$, which is unavailable in OP-UTVR.

\begin{figure*}[t!]
    \centering
    \includegraphics[width=.95\linewidth]{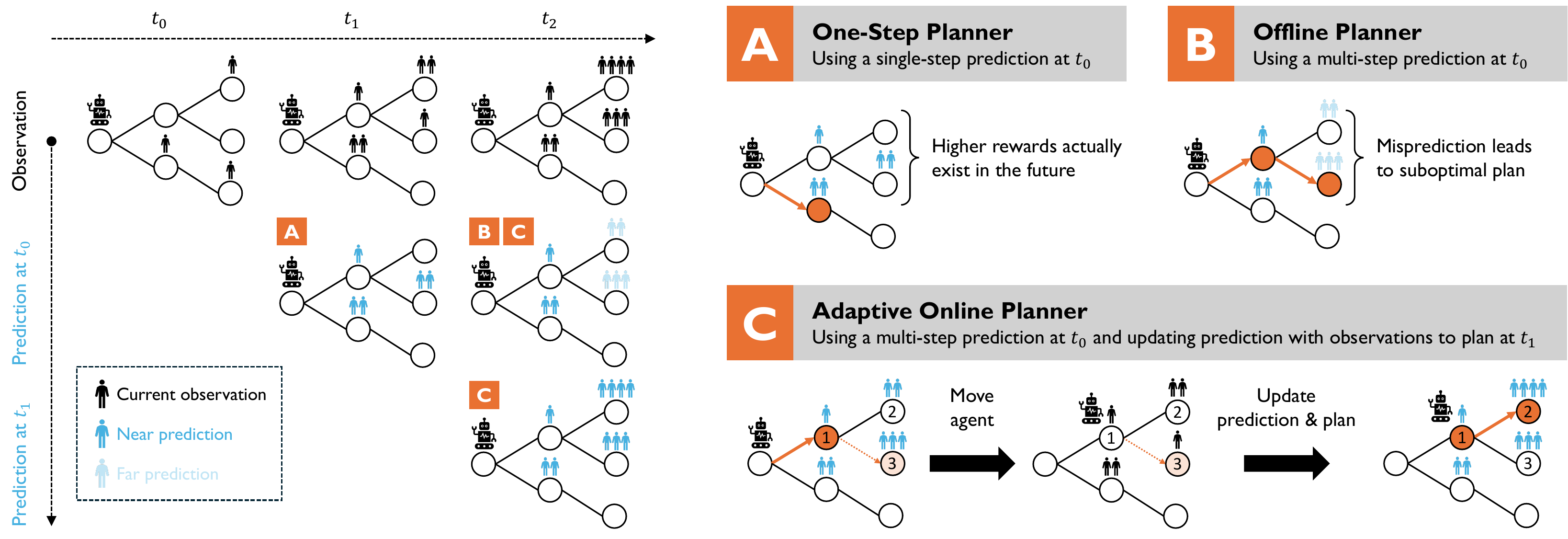}
    \caption{
    Solution overview. Left: observations (top row) refine predictions over time (lower rows). Badges A--C mark the information each planner uses. Right: (A) One-Step Planner maximizes immediate reward, (B) Offline Planner follows a precomputed multi-step plan, (C) Adaptive Online Planner updates both predictions and plans as observations arrive.
    }
    \label{fig:overview}
\end{figure*}

\section{Planning Algorithms}
\label{sec:solution}
As illustrated in Fig.~\ref{fig:overview}, we investigate three planning algorithms to solve OP-UTVR, named \emph{One-Step Planner, Offline Planner}, and \emph{Adaptive Online Planner}. 
These planners involve a policy $\pi$ that selects a next vertex, \ie, $v_{i+1}=\pi(v_i,...)$, and differ in how they utilize predicted time-varying rewards under uncertainty. 
In addition to these planners, Sec.~\ref{sec:analysis} provides novel performance guarantees under the stochastic reward dynamics of OP-UTVR.

\subsection{One-Step Planner}
One-Step Planner $\pi^\text{one-step}$ decides the next vertex to visit upon arrival at each vertex $v_i$, by greedily maximizing the predicted reward gained only at the next vertex without considering any future rewards (Fig.~\ref{fig:overview}A).
\begin{align}\label{eq:one-step}
    \pi^\text{one-step}(v_i, \mathbf{r}_\mathrm{tv}(T_i)) &= \argmax_{u \in \mathcal{V}} \hat{r}_\mathrm{tv}\left(u, d(e(v_i, u)) + T_i\right) \nonumber \\
    &= \argmax_{u \in \mathcal{V}} \bigl[ \mathbf{r}_\mathrm{tv}(T_i) \hat{P}^{d(e(v_i, u))} \bigr]_u,
\end{align}
where $[\mathbf{x}]_u$ denotes the entry of $\mathbf{x}$ corresponding to vertex $u$.

\subsection{Offline Planner}
Offline Planner $\pi^\mathrm{offline}$ computes a fixed plan in advance, \ie, a sequence of possibly overlapping vertices $V^*_0 = (v_0, v_1,...)$ maximizing the expected total reward under the time budget. 
Then, it traverses vertices following $V^*_0$, regardless of the observations during execution (Fig.~\ref{fig:overview}B).
\begin{align}
    \pi^\mathrm{offline}(v_i, T_i, \mathbf{r}_\mathrm{tv}(0)) &= v_{i+1},\nonumber\\
    \text{where} \;V^*_0 = (v_0, v_1,...) &= \argmax_{V} \mathbb{E}[\sum_{j=0}^{n-1} \hat{r}_\mathrm{tv}(v_j, T_j) | \mathbf{r}_\mathrm{tv}(0)] 
    \label{eq:offline}
\end{align}
Offline Planner thus relies on the estimate of Eq.~\eqref{eq:tv_model_t0} and returns an optimal solution for a traditional OP with known static rewards. 
However, it is not necessarily optimal for OP-UTVR as rewards move nondeterministically over time. 
In addition, the probabilistic transition model $\hat{P}$ is an estimate and is not the ground truth $P$. 
The estimation error accumulates over time steps, so the estimate would be increasingly unreliable over time.
These issues motivate us to deploy an online approach that can adapt during execution.

\subsection{Adaptive Online Planner}
The shortcoming of Offline Planner is that it only uses the observation of the reward $\mathbf{r}_\mathrm{tv}(0)$ at timestep $0$ to estimate the total reward using Eq.~\eqref{eq:tv_model_t0} and decide the entire plan despite observing the true reward ($\mathbf{r}_\mathrm{tv}(t)$) in the course of travel at timestep $t$.
Instead, Adaptive Online Planner \emph{replans} at each vertex arrival and estimates the total reward of the plan with the updated information of the reward ($\mathbf{r}_\mathrm{tv}(t)$) using Eq.~\eqref{eq:tv_model}:

\begin{align}
    \pi^\mathrm{adapt}(v_i, T_i, \mathbf{r}_\mathrm{tv}(t)) &= v_{i+1}, \nonumber\\
    \text{where} \;V^*_i = (v_i, v_{i+1},...) &= \argmax_{V} \mathbb{E}[\sum_{j=i}^{n-1} \hat{r}_\mathrm{tv}(v_j, T_j) | \mathbf{r}_\mathrm{tv}(T_{i})] 
    \label{eq:online}
\end{align}
Figure~\ref{fig:overview}C illustrates how Adaptive Online Planner works. 
Initially, the planner plans to visit \underline{vertex 1} while expecting to obtain a high reward at \underline{vertex 3} subsequently. 
Upon arrival at vertex 1, the planner updates the total reward ($R_\mathrm{adapt}$) with the current observation, and replans to visit \underline{vertex 2}, where the reward will be higher than in \underline{vertex 3}.

\subsection{Theoretical Analysis}
\label{sec:analysis}
We present the first analysis that quantifies how prediction errors and reward stochasticity affect each planner. See Appendix~\ref{sec:proofs} for the proofs.

First, Adaptive Online Planner achieves the optimal policy $\pi^*$ of the MDP in Lemma~\ref{lem:mdp} if $\hat{P}$ equals the true $P$.
\begin{theorem}
    If the estimated reward transition $\hat{P}$ has no error (\ie, $\hat{P} = P$),  Adaptive Online Planner achieves the optimal solution.
\end{theorem}
\noindent This result implies that the difficulty of OP-UTVR reduces to model error, as replanning itself loses no optimality.

The challenge of OP-UTVR is that the reward transition $P$ is not given to the agent, and it has to be learned from demonstrations.
Still, assuming the demonstrations are i.i.d. samples from the true distribution $P$, we can bound the suboptimality of Adaptive Online Planner as follows.
\begin{theorem}
    Assume the demonstrations are i.i.d. samples from $P$. Let $r_\mathrm{max}$ be the maximum reward obtainable at any vertex. Then, the expected total reward of Adaptive Online Planner using the empirical distribution as the estimated reward transition $\hat{P}$ is bounded suboptimal as:
    \begin{equation}
        |R_\mathrm{tot}(\pi^*) - R_\mathrm{tot}(\pi^\mathrm{adapt})| \leq 2r_\mathrm{max}T_\mathrm{max}\sqrt{\frac{\ln{(2|\mathcal{V}|^2/\delta)}}{2|D|}}
    \end{equation}
    with probability at least $1-\delta$, where $|D|$ is the number of transition samples.
\end{theorem}
\noindent The bound shrinks with more transition samples.
Reward transition estimation is thus not a fundamental obstacle in OP-UTVR, unlike the reward stochasticity itself.

The performance of Offline Planner compared to Adaptive Online Planner degrades when the reward transition $P$ has stochasticity, as a fixed plan commits to expected rewards.
\begin{theorem}
    Assume $\hat{P} = P$. Let $r_\mathrm{max}$ be the maximum reward obtainable at any vertex. Then,
    \begin{equation}
        R_\mathrm{tot}(\pi^\mathrm{adapt}) - R_\mathrm{tot}(\pi^\mathrm{offline}) \leq r_\mathrm{max} T_\mathrm{max} \cdot \frac{\sqrt{N \ln |\mathcal{V}|}}{2},
    \end{equation}
    where $N$ is the number of reward objects with nondeterministic transitions.
\end{theorem}
\noindent Here, a reward object is an individual reward source that moves between vertices according to $P$.
Theorem 3 suggests that the advantage of adaptation grows with reward stochasticity. 
Section \ref{sec:simulation_experiments} examines this relationship empirically.

\section{Experiments}
\label{sec:simulation_experiments}

We develop a mobile service robot simulation to benchmark OP-UTVR.
The simulation models a robot that navigates an environment to interact with people through activities such as displaying advertisements, gathering information, or monitoring pedestrians.
The robot's objective is to maximize the total number of interactions within a time budget.
The environment contains several distinct spots where people tend to congregate. 
While the occupancy of each spot can be observed continuously, \eg, via surveillance camera networks~\citep{fleuret2008multicamera}, individual trajectories are complex and difficult to predict over long horizons. 
This setting requires the robot to handle uncertain time-varying rewards.

\subsection{Mobile Service Robot Benchmark}
\label{sec:benchmark}

\paragraph{Environment construction.}

We build our benchmark on IR-SIM~\citep{han2026irsim}, an open-source multi-agent crowd navigation simulator.
Test environments come from the HM3D~\citep{ramakrishnan2021hm3d} and Gibson~\citep{xia2018gibson} datasets, which provide scans of real-world indoor facilities such as offices and stores (Fig.~\ref{fig:scanned_maps}). 
We evaluate planners on ten indoor environments, each 25 m $\times$ 25 m.
Each environment has 13 gathering spots, \ie, the vertices $\mathcal{V}$ of the OP-UTVR tuple.
Our formulation maps each pedestrian position to exactly one vertex (Sec.~\ref{sec:problem_formulation_formulation}).
We thus sample random points in the free space and define each gathering spot as the Voronoi region~\citep{aurenhammer2000voronoi} around each point (Fig.~\ref{fig:qualitative_results}).
For all pairs of gathering spots, we compute grid-based shortest paths using A$^*$ search, which are used during navigation for both the robot and pedestrians. 
We generate ten vertex placements per environment such that all spot pairs are connected, 100 test instances in total.

\paragraph{Robot modeling.}
The robot is omnidirectional with a maximum velocity of 1.2 m/s.
In planning, the travel time between spots, \ie, $d$ of the OP-UTVR tuple, is computed assuming the robot moves along the shortest path at maximum velocity. 
At execution time, the robot follows the path to the target spot but uses the reciprocal velocity obstacles (RVO) model~\citep{berg2008reciprocal} to avoid collisions with pedestrians.
Thus, $d$ only approximates the true travel time, which induces uncertainty as in practical robot navigation.
The simulation runs in discrete steps of $\Delta t = 0.5$ s, and each step updates robot control and pedestrian states.

\paragraph{Pedestrian modeling.} 
Each environment contains 20 pedestrians that transition stochastically between gathering spots according to a Markov chain with transition matrix $P$.
We design $P$ by randomly selecting one spot as the ``entrance'' where all pedestrians start.
The remaining spots are evenly partitioned into two loops, and pedestrians enter either with probability 0.5 from the entrance.
Within a loop, pedestrians move to the next spot with probability 0.75, skip one spot ahead with probability 0.05, or remain at the current spot with probability 0.2.
This design creates stochastic crowd behaviors conditioned on latent intent (\ie, which group each pedestrian chooses).
During execution, pedestrians move to their target spot along the pre-calculated shortest path at 1.0 m/s and use RVO to avoid collisions with other pedestrians and the robot. 
Interactions with the robot may change pedestrian behaviors from their transition patterns.

\paragraph{Transition matrix estimation.}
To predict future rewards, the robot estimates $\hat{P}$ from observed pedestrian trajectories.
For each test instance, we collect 240 s of trajectory data without robot presence.
Although these data follow the underlying rule defined by $P$, individual trajectories vary due to stochasticity.
We estimate $\hat{P}$ by fitting the gathered trajectory data using least-squares optimization.

\begin{figure}[!t]
  \centering
  {\makebox[0.5\columnwidth]{\includegraphics[width=0.45\columnwidth]{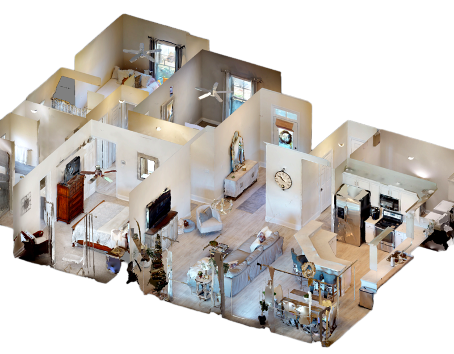}}\label{fig:sub1}}%
  {\makebox[0.5\columnwidth]{\includegraphics[width=0.45\columnwidth]{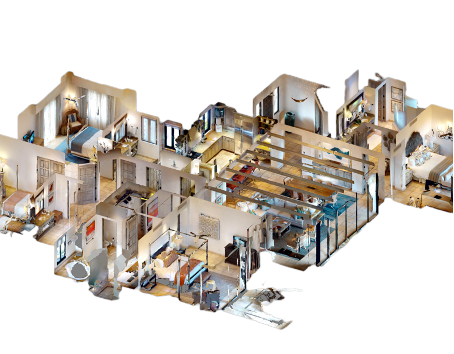}}\label{fig:sub2}}
  \caption{3D views of example indoor environments from HM3D. We use bird's-eye view maps from these scans.}
  \label{fig:scanned_maps}
\end{figure}

\begin{figure*}[t!]
    \centering
    \includegraphics[width=.95\linewidth]{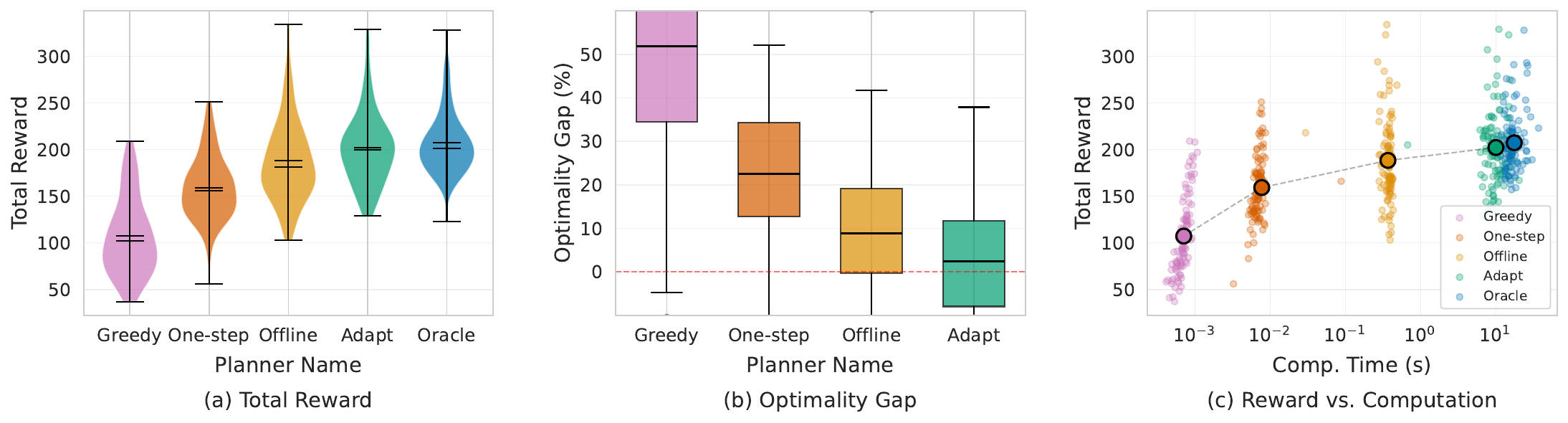}
    \caption{Statistical performance comparison of five planners across 100 instances (10 vertex placements $\times$ 10 maps). (a) Total reward. (b) Optimality gap relative to $\pi^{\mathrm{oracle}}$. (c) Reward-computation tradeoff, with points as instances and circles as means.}
    \label{fig:quantitative_results}
\end{figure*}

\paragraph{Reward dynamics modeling.}

The reward $r_\mathrm{tv}(v, t)$ at each vertex is the number of pedestrians within the associated Voronoi region at time $t$.
As described in Sec.~\ref{sec:ot_utvr_solution}, we allow $d(e(v,v)) > 0$ so that the robot can remain at a vertex.
We set $d(e(v,v)) = 5$ s for all $v\in\mathcal{V}$, and the robot collects $r_\mathrm{tv}(v_i, t_i)$ upon each arrival.

\subsection{Evaluation Setup}
\paragraph{Planner implementation.}

For Offline Planner and Adaptive Online Planner, we compute solution paths using A$^*$ search~\citep{hart1968formal} over the \emph{expected} rewards estimated via $\hat{P}^t$, with rewards treated as negative costs. 
For each partial path, branch-and-bound pruning~\citep{land1960automatic} computes an upper bound on the remaining expected reward.
This bound is used as the admissible heuristic of A$^*$, and branches whose bound cannot exceed the best solution are pruned.
Adaptive Online Planner re-executes this search at each arrival. 
In addition to the three planners, we evaluate two baselines: 1) Greedy Planner ($\pi^\mathrm{greedy}$), which selects the next vertex with the highest current pedestrian count without predicting future rewards, and 2) Oracle Planner ($\pi^\mathrm{oracle}$), which uses future pedestrian distributions from collected trajectory data and provides a coarse upper bound.
$\pi^\mathrm{oracle}$ replaces the optimal MDP policy (Lemma~\ref{lem:mdp}), which requires the true $P$ and is not achievable in practice.

\paragraph{Evaluation metrics.}

The following metrics assess each planner's solution quality and computational efficiency:
\begin{itemize}
    \item \textbf{Total reward} ($R_{\mathrm{tot}}$) is the cumulative reward collected during execution, \ie, pedestrian encounters within $T_{\mathrm{max}}$.
    \item \textbf{Optimality gap} ($\Delta_{\mathrm{opt}}$) [\%] is computed as $(R_{\mathrm{oracle}} - R_{\mathrm{tot}}) / R_{\mathrm{oracle}} \times 100$, where $R_{\mathrm{oracle}}$ is the total reward of $\pi^{\mathrm{oracle}}$. This measures proximity to $\pi^{\mathrm{oracle}}$, though it is not a strict upper bound due to execution uncertainty.
    \item \textbf{Computation time} ($CT$) [s] measures computational cost as total planning time in the course of execution.
\end{itemize}

\subsection{Results}

\paragraph{Quantitative results.}

Figure~\ref{fig:quantitative_results} summarizes the performance of five planners across 100 test instances.
$\pi^{\mathrm{oracle}}$ achieves the highest total reward because it plans with true future pedestrian distributions (Fig.~\ref{fig:quantitative_results}(a)). 
$\pi^{\mathrm{greedy}}$, which relies on current observations, achieves the lowest total reward. 
All three proposed planners that leverage predicted reward dynamics outperform this baseline. 
This result validates our formulation, as planning with predicted reward dynamics is effective even when true dynamics remain unknown.

Among the three planners, $\pi^{\mathrm{adapt}}$ achieves the highest total reward in 64 cases, compared to 25 for $\pi^{\mathrm{offline}}$ and 11 for $\pi^{\text{one-step}}$.
Figure~\ref{fig:quantitative_results}(b) shows their optimality gaps relative to $\pi^{\mathrm{oracle}}$.
$\pi^{\text{one-step}}$ has a 22.5\% gap due to myopic decision-making, despite using updated observations.
$\pi^{\mathrm{offline}}$ reduces the gap to 8.9\% by planning over the full time horizon, but its fixed plan cannot adapt to growing prediction errors.
$\pi^{\mathrm{adapt}}$ approaches the oracle with a 2.1\% gap, as replanning keeps predictions up to date and prevents error accumulation (pairwise Wilcoxon signed-rank, Holm-corrected $p<0.001$).
These results suggest that effective planning combines long-horizon planning with runtime prediction updates.

\begin{figure*}[t!]
    \centering
    \includegraphics[width=.95\linewidth]{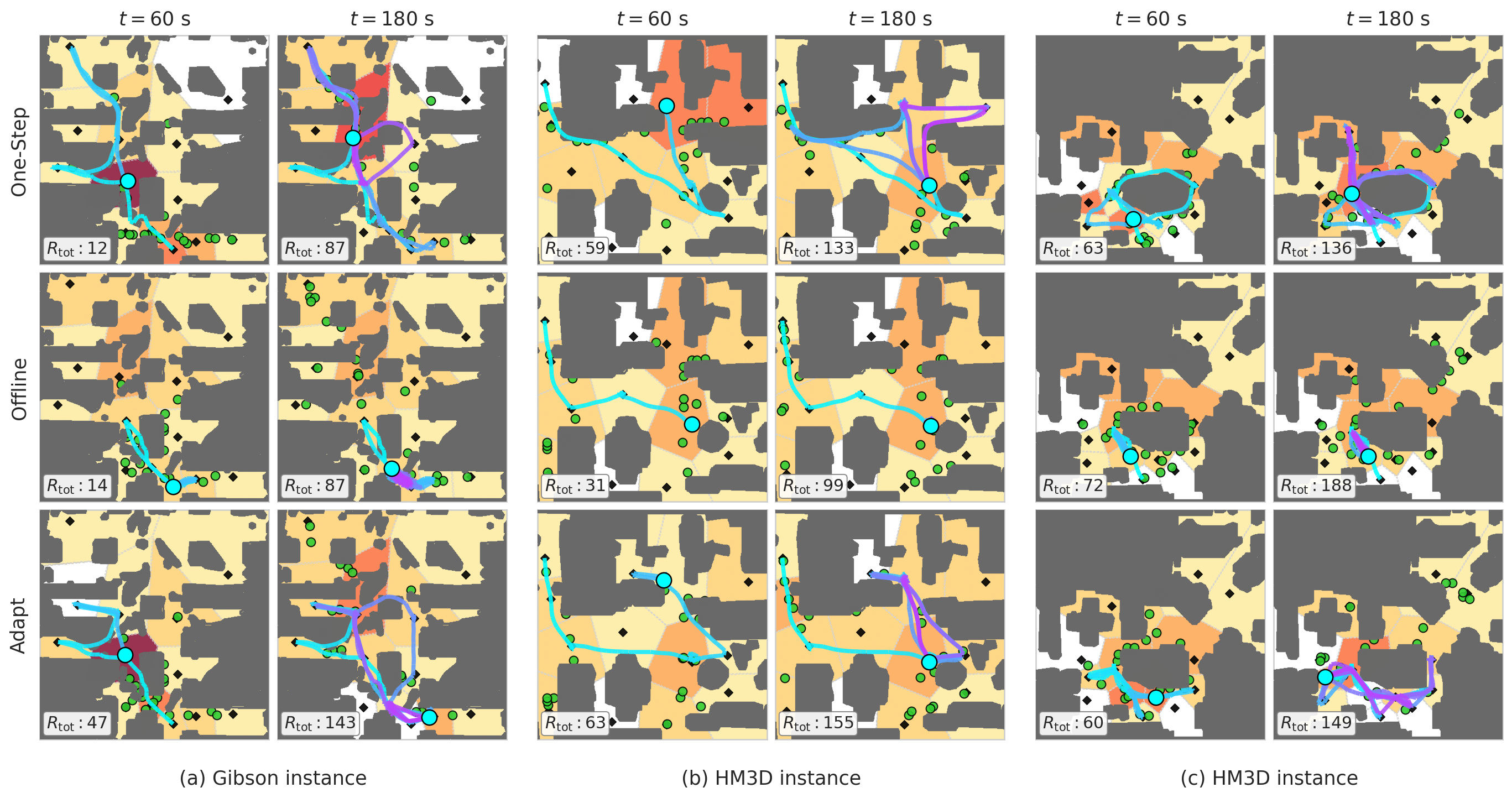}
    \caption{Snapshots of three planners for robot orienteering among pedestrians.
    $R_\mathrm{tot}$ denotes the total reward at each snapshot.
    Voronoi cell colors indicate predicted rewards (warmer for higher values), with gray dashed boundaries. 
    Annotations: black diamonds (graph vertices), dark gray (obstacles), cyan dot (robot), trajectory cyan-to-magenta over time, green dots (pedestrians).}
    \label{fig:qualitative_results}
\end{figure*}

These performance differences come with tradeoffs in computational costs (Fig.~\ref{fig:quantitative_results}(c)). 
$\pi^{\text{one-step}}$ requires minimal computation as it only evaluates the immediate rewards of candidate next vertices.
$\pi^{\mathrm{offline}}$ involves moderate computation by solving the planning problem once at the start. 
$\pi^{\mathrm{adapt}}$ demands the highest computation due to replanning, but achieves near-oracle performance.
Beyond these general trends, we next examine individual planner behaviors.

\paragraph{Qualitative results.}

Figure~\ref{fig:qualitative_results} compares trajectories of the three planners across three test instances, with snapshots at $t = 60$ and $180$ s.
In Fig.~\ref{fig:qualitative_results}(a), $\pi^{\mathrm{adapt}}$ achieves a total reward of 184, while $\pi^{\mathrm{offline}}$ and $\pi^{\text{one-step}}$ obtain only 118 and 121, respectively (final values at $t = 240$ s; Fig.~\ref{fig:qualitative_results} shows snapshots).
$\pi^{\mathrm{offline}}$ repeatedly visits a few vertices, as its initial plan cannot adapt to reward dynamics.
Its predictions also converge over time, so moving offers little expected gain.
$\pi^{\text{one-step}}$ responds to current observations but fails to collect high $R_{\mathrm{tot}}$ due to optimizing only the next step.
$\pi^{\mathrm{adapt}}$ visits diverse vertices where pedestrians actually gather, by replanning with updated observations while optimizing over multiple steps.

\begin{figure}[t!]
    \centering
    \includegraphics[width=.95\linewidth]{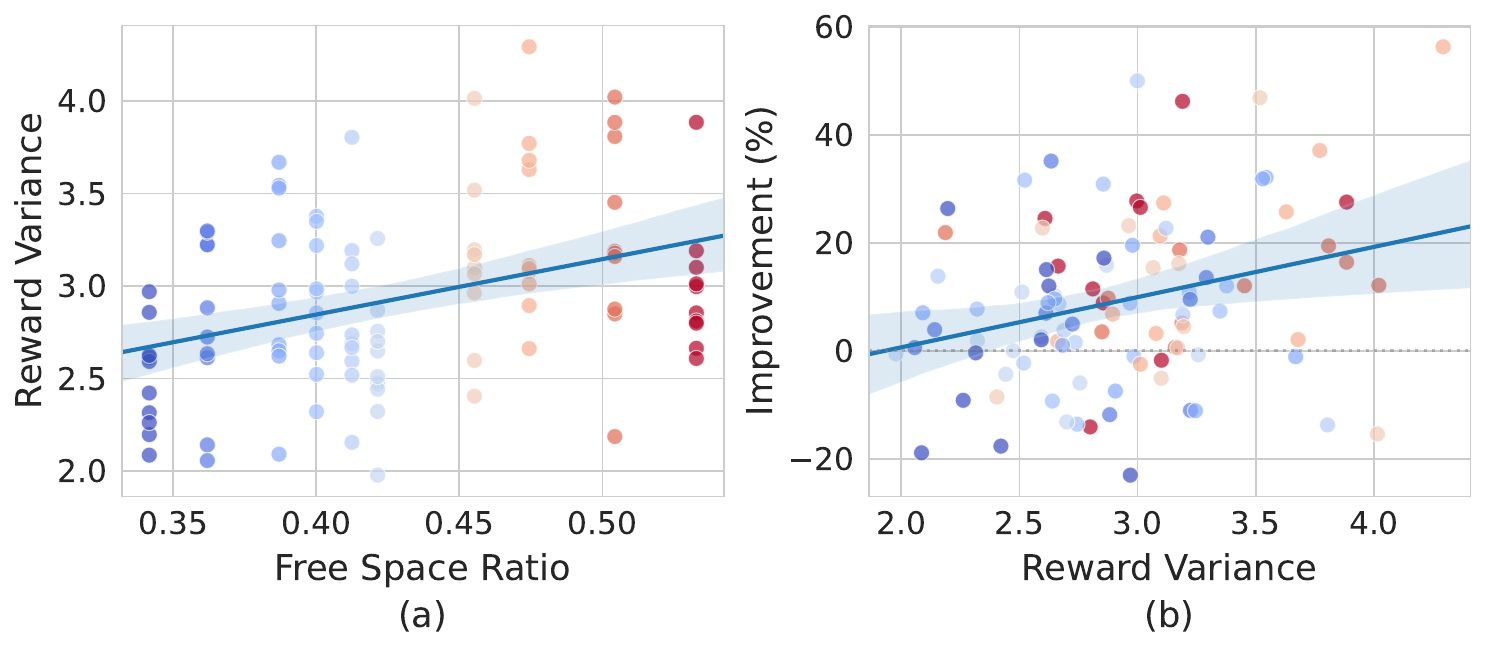}
    \caption{Environmental effect on planner performance.
    Point colors indicate maps.  
    (a) Free space ratio vs. reward variance. (b) Reward variance vs. improvement of $\pi^{\mathrm{adapt}}$ over $\pi^{\mathrm{offline}}$.}
    \label{fig:env_analysis}
\end{figure}

\paragraph{Analysis of environmental effects.}

We quantify the relationship between environment geometry and planner performance for all instances using two metrics: free space ratio (0.34--0.53 across maps) and reward variance (temporal variance per vertex). 
Figure~\ref{fig:env_analysis} shows that higher free space ratio correlates with greater reward variance (Pearson correlation, $r = 0.36$, $p < 0.001$), which increases the improvement of $\pi^{\mathrm{adapt}}$ over $\pi^{\mathrm{offline}}$ ($r = 0.29$, $p < 0.01$, measured as $(R_{\mathrm{adapt}} - R_{\mathrm{offline}}) / R_{\mathrm{offline}} \times 100$ [\%]).
Open spaces allow pedestrians to move freely, and rewards vary more over time.
$\pi^{\mathrm{offline}}$ follows a fixed plan and cannot adapt to these shifts, while $\pi^{\mathrm{adapt}}$ updates its plan with new observations.
Figure~\ref{fig:qualitative_results}(b) and (c) illustrate this trend: $\pi^{\mathrm{adapt}}$ leads in open spaces (194 vs. 135) while $\pi^{\mathrm{offline}}$ outperforms in constrained spaces (241 vs. 193) (final values at $t = 240$ s).

\section{Related Work}
\label{sec:related_work}

\paragraph{OP and its variants.} 
The OP selects and sequences vertex visits to maximize collected rewards under budget constraints~\citep{tsiligirides1984heuristic}.
As an NP-hard problem, OP has been extensively studied in operations research in terms of optimality~\citep{fischetti1998solving}, computational efficiency~\citep{tang2005tabu,wang2008using}, and scalability~\citep{boussier2007exact,dang2013effective}.
Its abstract formulation readily extends to variants such as time-window constraints~\citep{kantor1992orienteering} and team orienteering~\citep{chao1996team}.

\paragraph{Robotic applications.}
This flexibility makes OP well-suited for robotics, where agents maximize task completion under resource constraints such as time or energy.
For example, OP is extended to capture spatial correlations for environmental monitoring~\citep{yu2016correlated}, as nearby locations provide similar measurements.
\citet{peltzer2022figop} incorporates information gain into OP to prioritize areas that reduce map uncertainty for unknown terrain exploration.
For search-and-rescue operations~\citep{jorgensen2024matroid}, OP introduces survival constraints to account for the risk of agent loss.
These examples show how robotics requires extending OP beyond its classical formulation.

\paragraph{Uncertain/dynamic rewards.}
While most OP formulations assume static or known rewards, the real world is inherently uncertain and dynamic.
Prior work addresses this through stochastic formulations or time-varying reward models.
The former handle uncertainty by modeling stochastic travel costs with chance constraints~\citep{carpin2025solving} or rewards with probability distributions~\citep{Ilhan2008orienteering}.
However, these approaches assume that uncertainty follows known, stationary distributions.
In contrast, time-varying reward models capture rewards that depend on location and visit time.
\citet{ma2017spatio} introduced this formulation and solved it via dynamic programming on a spatio-temporal graph.
\citet{cao2024heuristic} improved computational efficiency with a heuristic search that guarantees optimality without explicit state spaces.
Yet both methods require complete knowledge of future reward dynamics, an assumption rarely satisfied in practice.
We combine both directions by considering time-varying rewards with uncertain dynamics.

\paragraph{Relation to MDP.}
When future reward dynamics are unknown, one might formulate the problem as an MDP~\citep{puterman2014markov}.
While prior work has studied MDPs with changing rewards~\citep{cardoso2019large} or non-stationary dynamics~\citep{cheung2020reinforcement}, such methods require learning both state transitions and reward functions.
In our setting, agent transitions are deterministic and known, but only reward dynamics are uncertain. 
This allows us to learn only reward transitions rather than the entire MDP, which simplifies the learning problem.

\section{Conclusion} 
\label{sec:conclusion}

We presented OP-UTVR, a novel variant of the orienteering problem where rewards change stochastically and cannot be known in advance.
Unlike prior work assuming known reward dynamics, our formulation allows agents to predict future reward transitions from observations.
We investigated three planning algorithms and provided theoretical bounds on their performance.
We also introduced a mobile service robot benchmark using real-world 3D scans.
Experiments revealed trade-offs between planning horizon and adaptivity, and demonstrated the effectiveness of combining long-horizon planning with runtime observation updates.
While we focused on Markovian dynamics in simulation, future directions include handling non-Markovian transitions, scaling to larger graphs, and validating with physical robots.

\bibliography{aaai2027}

\appendix

\setcounter{theorem}{0}
\setcounter{lemma}{0}

\section{Proofs}
\label{sec:proofs}
We prove Theorems 1--3 and the supporting lemmas for the orienteering problem with uncertain time-varying rewards (OP-UTVR).

\subsection{Proof of Theorem 1}
\begin{theorem}
    If the estimated reward transition $\hat{P}$ has no error (\ie, $\hat{P} = P$), Adaptive Online Planner achieves the optimal solution.
\end{theorem}

\begin{proof}
To show Theorem 1, we first prove the following lemma.
\begin{lemma}\label{lm:mdp}
    Assume $\hat{P} = P$. An instance of OP-UTVR is an instance of a Markov decision process (MDP).
\end{lemma}

\begin{proof}[Proof of Lemma~\ref{lm:mdp}]
    An MDP consists of a tuple $(S, A, T, R)$~\citep{puterman2014markov}. An MDP representing an instance of OP-UTVR can be constructed as follows.
    
    \textbf{State space $S$:} A state consists of the agent's position $v \in \mathcal{V}$, the current reward distribution $\mathbf{r}_\mathrm{tv} \in \mathbb{R}_{\geq}^{1 \times |\mathcal{V}|}$, and the remaining time budget $t_\mathrm{remain} \in [0, T_\mathrm{max}]$: 
    $$s = (v, \mathbf{r}_\mathrm{tv}, t_\mathrm{remain}).$$
    
    \textbf{Action space $A$:} The action space has size $|\mathcal{V}|$, where each action corresponds to selecting a vertex to travel to next.
    
    \textbf{Transition function $T$:} Given state $s = (v, \mathbf{r}_\mathrm{tv}, t_\mathrm{remain})$ and action $a \in \mathcal{V}$, the next state $s' = (v', \mathbf{r}'_\mathrm{tv}, t'_\mathrm{remain})$ is determined as follows:
    \begin{align}
        v' &= a, \\
        \mathbf{r}'_\mathrm{tv} &\sim \mathbf{r}_\mathrm{tv} \, P^{d(e(v,a))}, \\
        t'_\mathrm{remain} &= t_\mathrm{remain} - d(e(v, a)).
    \end{align}
    The agent's position is deterministically updated to $v' = a$. The reward distribution evolves probabilistically according to the transition matrix $P$ raised to the power of the travel time $d(e(v, a))$. The remaining time is decremented by the travel time. If $t'_\mathrm{remain} \leq 0$, the system transitions to an absorbing terminal state.
    
    \textbf{Reward function $R$:} The agent receives reward $r_\mathrm{tv}(v, t)$ upon arriving at vertex $v$ at timestep $t$.
    
    The policy of this MDP corresponds exactly to the policy in OP-UTVR, and the expected total rewards of the two formulations coincide.
    Therefore, OP-UTVR is an instance of an MDP.
\end{proof}

By Lemma~\ref{lm:mdp}, when $\hat{P} = P$, OP-UTVR is an MDP with known transition dynamics. 
The optimal policy of an MDP selects, at each state, the action that maximizes the expected total reward.
It can be obtained via dynamic programming or policy iteration.
Adaptive Online Planner defined in Sec.~\ref{sec:solution} (Eq.~\eqref{eq:online}) computes exactly this: at each state, it selects the action (next vertex) that maximizes the expected total reward given the current observations. Therefore, Adaptive Online Planner achieves the optimal solution when $\hat{P} = P$, which proves Theorem~1.
\end{proof}

\subsection{Proof of Theorem 2}
\begin{theorem}
    Assume the demonstrations are i.i.d. samples from $P$. Let $r_\mathrm{max}$ be the maximum reward obtainable at any vertex. Then, the expected total reward of Adaptive Online Planner using the empirical distribution as the estimated reward transition $\hat{P}$ is bounded suboptimal as:
    \begin{equation}
        |R_\mathrm{tot}(\pi^*) - R_\mathrm{tot}(\pi^\mathrm{adapt})| \leq 2r_\mathrm{max}T_\mathrm{max}\sqrt{\frac{\ln{(2|\mathcal{V}|^2/\delta)}}{2|D|}}
    \end{equation}
    with probability at least $1-\delta$, where $|D|$ is the number of transition samples.
\end{theorem}

\begin{proof}
To show Theorem 2, we first prove the following lemma.

\begin{lemma}\label{lm:epsilon}
    Suppose $\|\hat{P} - P\|_{\infty} \leq \epsilon$, where $\|\cdot\|_{\infty}$ denotes the max-norm: $\|M\|_{\infty} = \max_{u,v} |M_{u,v}|$. Let $r_\mathrm{max}$ be the maximum reward obtainable at any vertex. Then,
    \begin{equation}
        |R_\mathrm{tot}(\pi^*) - R_\mathrm{tot}(\pi^\mathrm{adapt})| \leq 2\epsilon r_\mathrm{max} T_\mathrm{max}.
    \end{equation}
\end{lemma}

\begin{proof}[Proof of Lemma~\ref{lm:epsilon}]
Let $V^{\pi}_P(s)$ denote the expected total reward when following policy $\pi$ under true dynamics $P$, starting from state $s$. Similarly, let $V^{\pi}_{\hat{P}}(s)$ denote the expected total reward under estimated dynamics $\hat{P}$.

By the Bellman equation for finite-horizon MDPs, for any policy $\pi$ and state $s = (v, \mathbf{r}_\mathrm{tv}, t_\mathrm{remain})$ with $t_\mathrm{remain} > 0$:
\begin{align}
    V^{\pi}_P(s) &= r_\mathrm{tv}(v) + \mathbb{E}_{s' \sim T_P(s, \pi(s))}[V^{\pi}_P(s')], \\
    V^{\pi}_{\hat{P}}(s) &= r_\mathrm{tv}(v) + \mathbb{E}_{s' \sim T_{\hat{P}}(s, \pi(s))}[V^{\pi}_{\hat{P}}(s')],
\end{align}
where $T_P$ and $T_{\hat{P}}$ are the transition functions under $P$ and $\hat{P}$, respectively.

The difference in value functions can be bounded recursively. For a single timestep:
\begin{align}
    |V^{\pi}_P(s) - V^{\pi}_{\hat{P}}(s)| &\leq \left|\mathbb{E}_{s' \sim T_P}[V^{\pi}_P(s')] - \mathbb{E}_{s' \sim T_{\hat{P}}}[V^{\pi}_{\hat{P}}(s')]\right| \\
    &\leq \left|\mathbb{E}_{s' \sim T_P}[V^{\pi}_P(s')] - \mathbb{E}_{s' \sim T_P}[V^{\pi}_{\hat{P}}(s')]\right| \\
    &\quad + \left|\mathbb{E}_{s' \sim T_P}[V^{\pi}_{\hat{P}}(s')] - \mathbb{E}_{s' \sim T_{\hat{P}}}[V^{\pi}_{\hat{P}}(s')]\right|.
\end{align}

The first term bounds the recursive error in value estimation. The second term captures the error due to transition mismatch. Since rewards are bounded by $r_\mathrm{max}$, the maximum value function is bounded by $r_\mathrm{max} T_\mathrm{max}$. The transition error $\|\hat{P} - P\|_{\infty} \leq \epsilon$ implies that the distribution over next states differs by at most $\epsilon$ in total variation distance. Therefore:
\begin{equation}
    \left|\mathbb{E}_{s' \sim T_P}[V^{\pi}_{\hat{P}}(s')] - \mathbb{E}_{s' \sim T_{\hat{P}}}[V^{\pi}_{\hat{P}}(s')]\right| \leq \epsilon r_\mathrm{max} T_\mathrm{max}.
\end{equation}

Unrolling this recursion over $T_\mathrm{max}$ steps yields:
\begin{equation}
    |V^{\pi}_P(s) - V^{\pi}_{\hat{P}}(s)| \leq \epsilon r_\mathrm{max} T_\mathrm{max}.
\end{equation}

Now, $\pi^*$ is optimal under $P$, and $\pi^\mathrm{adapt}$ is optimal under $\hat{P}$ (by Theorem~1). Therefore:
\begin{align}
    R_\mathrm{tot}(\pi^*) - &R_\mathrm{tot}(\pi^\mathrm{adapt}) = V^{\pi^*}_P(s_0) - V^{\pi^\mathrm{adapt}}_P(s_0) \\
    &\leq V^{\pi^*}_P(s_0) - V^{\pi^\mathrm{adapt}}_{\hat{P}}(s_0) \nonumber\\
    &+ V^{\pi^\mathrm{adapt}}_{\hat{P}}(s_0) - V^{\pi^\mathrm{adapt}}_P(s_0) \\
    &\leq V^{\pi^*}_P(s_0) - V^{\pi^\mathrm{adapt}}_{\hat{P}}(s_0) + \epsilon r_\mathrm{max} T_\mathrm{max}.
\end{align}

Since $\pi^\mathrm{adapt}$ is optimal under $\hat{P}$, we have $V^{\pi^\mathrm{adapt}}_{\hat{P}}(s_0) \geq V^{\pi^*}_{\hat{P}}(s_0)$. Thus:
\begin{align}
    V^{\pi^*}_P(s_0) - V^{\pi^\mathrm{adapt}}_{\hat{P}}(s_0) &\leq V^{\pi^*}_P(s_0) - V^{\pi^*}_{\hat{P}}(s_0) \leq \epsilon r_\mathrm{max} T_\mathrm{max}.
\end{align}

Combining these bounds:
\begin{equation}
    |R_\mathrm{tot}(\pi^*) - R_\mathrm{tot}(\pi^\mathrm{adapt})| \leq 2\epsilon r_\mathrm{max} T_\mathrm{max}.
\end{equation}
\end{proof}

Let $\hat{P}$ be the empirical transition matrix estimated from $|D|$ i.i.d. transition samples. For each pair $(u, v) \in \mathcal{V} \times \mathcal{V}$, the empirical estimate $\hat{P}_{u,v}$ is the sample mean of Bernoulli trials (whether a transition from $u$ lands at $v$). By Hoeffding's inequality~\citep{hoeffding1963probability}, for any fixed $(u, v)$:
\begin{equation}
    \Pr(|P_{u,v} - \hat{P}_{u,v}| > \epsilon) \leq 2 \exp(-2|D|\epsilon^2).
\end{equation}

Applying a union bound over all $|\mathcal{V}|^2$ entries:
\begin{equation}
    \Pr(\|\hat{P} - P\|_{\infty} > \epsilon) \leq 2|\mathcal{V}|^2 \exp(-2|D|\epsilon^2).
\end{equation}

Setting the right-hand side equal to $\delta$ and solving for $\epsilon$:
\begin{align}
    2|\mathcal{V}|^2 \exp(-2|D|\epsilon^2) &= \delta \\
    \epsilon &= \sqrt{\frac{\ln(2|\mathcal{V}|^2/\delta)}{2|D|}}.
\end{align}

Therefore, with probability at least $1-\delta$:
\begin{equation}
    \|\hat{P} - P\|_{\infty} \leq \sqrt{\frac{\ln(2|\mathcal{V}|^2/\delta)}{2|D|}}.
\end{equation}

Combining with Lemma~\ref{lm:epsilon}, we obtain:
\begin{equation}
    |R_\mathrm{tot}(\pi^*) - R_\mathrm{tot}(\pi^\mathrm{adapt})| \leq 2r_\mathrm{max}T_\mathrm{max}\sqrt{\frac{\ln(2|\mathcal{V}|^2/\delta)}{2|D|}},
\end{equation}
which completes the proof of Theorem~2.
\end{proof}

\subsection{Proof of Theorem 3}
\begin{theorem}
    Assume $\hat{P} = P$. Let $r_\mathrm{max}$ be the maximum reward obtainable at any vertex. Then,
    \begin{equation}
        R_\mathrm{tot}(\pi^\mathrm{adapt}) - R_\mathrm{tot}(\pi^\mathrm{offline}) \leq r_\mathrm{max} T_\mathrm{max} \cdot \frac{\sqrt{N \ln |\mathcal{V}|}}{2},
    \end{equation}
    where $N$ is the number of reward objects with nondeterministic transitions.
\end{theorem}

\begin{proof}
Let $\{X_v\}_{v \in \mathcal{V}}$ be random variables representing rewards at different vertices at some future timestep.
\begin{lemma}\label{lm:exp_max_gap}
For Adaptive Online Planner, the expected reward per step is $\mathbb{E}[\max_{v \in \mathcal{V}} X_v]$. For Offline Planner, the expected reward per step is at most $\max_{v \in \mathcal{V}} \mathbb{E}[X_v]$. Therefore:
\begin{equation}
    \mathbb{E}\left[\max_{v \in \mathcal{V}} X_v\right] - \max_{v \in \mathcal{V}} \mathbb{E}[X_v] \geq 0.
\end{equation}
\end{lemma}

\begin{proof}[Proof of Lemma~\ref{lm:exp_max_gap}]
Let $v^* = \argmax_{v} \mathbb{E}[X_v]$. Then:
\begin{align}
    \mathbb{E}\left[\max_{v \in \mathcal{V}} X_v\right] 
    &\geq \mathbb{E}[X_{v^*}] \\
    &= \max_{v \in \mathcal{V}} \mathbb{E}[X_v].
\end{align}
\end{proof}

Assume there are $N$ independent reward objects, each following stochastic transitions according to $P$. Each object contributes a reward in $[0, r_\mathrm{max}]$ when collected.

At any timestep $t$, let $Y_i^{(v)} \in [0, r_\mathrm{max}]$ denote the reward from object $i$ if it is located at vertex $v$, and $0$ otherwise. The total reward at vertex $v$ is:
\begin{equation}
    X_v = \sum_{i=1}^{N} Y_i^{(v)}.
\end{equation}

Since each object transitions independently according to $P$, the $Y_i^{(v)}$ are independent random variables. Let $p_v = \Pr(\text{object } i \text{ is at vertex } v)$ be the probability that object $i$ is at vertex $v$. Then:
\begin{equation}
    \mathbb{E}[X_v] \leq N r_\mathrm{max} p_v.
\end{equation}

For the maximum of $|\mathcal{V}|$ independent sums (one per vertex), we have the following lemma using a concentration inequality.
\begin{lemma}\label{lm:max_concentration}
For independent random variables bounded in $[0, N r_\mathrm{max}]$, the gap between the expectation of the maximum and the maximum of expectations satisfies:
\begin{equation}
    \mathbb{E}\left[\max_{v \in \mathcal{V}} X_v\right] - \max_{v \in \mathcal{V}} \mathbb{E}[X_v] \leq r_\mathrm{max} \sqrt{\frac{N \ln |\mathcal{V}|}{2}}.
\end{equation}
\end{lemma}

\begin{proof}[Proof of Lemma~\ref{lm:max_concentration}]
Let $v^* = \argmax_v \mathbb{E}[X_v]$ and let $\mu = \mathbb{E}[X_{v^*}]$. For any vertex $v$, by Hoeffding's inequality applied to the sum $X_v = \sum_{i=1}^N Y_i^{(v)}$ where each $Y_i^{(v)} \in [0, r_\mathrm{max}]$:
\begin{equation}
    \Pr(X_v - \mathbb{E}[X_v] \geq t) \leq \exp\left(-\frac{2t^2}{N r_\mathrm{max}^2}\right).
\end{equation}

Therefore:
\begin{align}
    \Pr\left(\max_v X_v \geq \mu + t\right) &\leq \sum_{v \in \mathcal{V}} \Pr(X_v \geq \mu + t) \\
    &\leq \sum_{v \in \mathcal{V}} \Pr(X_v - \mathbb{E}[X_v] \geq t) \\
    &\leq |\mathcal{V}| \exp\left(-\frac{2t^2}{N r_\mathrm{max}^2}\right).
\end{align}

Let $Z = \max_v X_v - \mu$. Then:
\begin{align}
    \mathbb{E}[Z] &= \int_0^\infty \Pr(Z \geq t) \, dt \\
    &\leq \int_0^\infty \min\left(1, |\mathcal{V}| \exp\left(-\frac{2t^2}{N r_\mathrm{max}^2}\right)\right) dt.
\end{align}

Setting $t_0 = r_\mathrm{max}\sqrt{\frac{N \ln |\mathcal{V}|}{2}}$ as the point where the exponential becomes small:
\begin{align}
    \mathbb{E}[Z] &\leq t_0 + \int_{t_0}^\infty |\mathcal{V}| \exp\left(-\frac{2t^2}{N r_\mathrm{max}^2}\right) dt \\
    &\leq r_\mathrm{max}\sqrt{\frac{N \ln |\mathcal{V}|}{2}} + \mathcal{O}(1) \\
    &\leq r_\mathrm{max}\sqrt{\frac{N \ln |\mathcal{V}|}{2}} \quad \text{(for sufficiently large $|\mathcal{V}|$)}.
\end{align}
\end{proof}

By Lemma~\ref{lm:max_concentration}, the gap per timestep is at most $r_\mathrm{max} \cdot \sqrt{N \ln |\mathcal{V}|}/2$.
Over $T_\mathrm{max}$ timesteps, the total gap is bounded by:
\begin{equation}
    R_\mathrm{tot}(\pi^\mathrm{adapt}) - R_\mathrm{tot}(\pi^\mathrm{offline}) \leq r_\mathrm{max} T_\mathrm{max} \cdot \frac{\sqrt{N \ln |\mathcal{V}|}}{2}.
\end{equation}
\end{proof}

\begin{corollary}
    If $P$ is deterministic, then 
    \begin{equation}
        R_\mathrm{tot}(\pi^\mathrm{adapt}) = R_\mathrm{tot}(\pi^\mathrm{offline}).
    \end{equation}
\end{corollary}

\begin{proof}
With $P$ being deterministic, the value of $X_v$ is fixed for any trials. Thus, $\mathbb{E}[\max_v X_v] = \max_v \mathbb{E}[X_v]$.
\end{proof}

\section{Experimental Details}
\label{sec:exp_details}
All experiments were run on CPUs (8~CPU cores, 30\,GB RAM) and require no GPU. The implementation is written in Python~3.12, and exact dependency versions are included with our code.
Table~\ref{tab:params} lists the parameters used in our experiments; see Sec.~\ref{sec:benchmark} for the design of the transition matrix $P$.
All parameters were fixed a priori as benchmark design choices rather than tuned via hyperparameter search. 
These values are based on realistic service-robot settings, such as the robot and pedestrian velocities.

\begin{table}[h]
\centering
\caption{Experimental parameters.}
\label{tab:params}
\begin{tabular}{ll}
\toprule
Parameter & Value \\
\midrule
Environment size & 25\,m $\times$ 25\,m \\
Number of environments (maps) & 10 \\
Vertex placements per environment & 10 \\
Gathering spots per environment ($|\mathcal{V}|$) & 13 \\
Number of pedestrians ($N$) & 20 \\
Time budget ($T_\mathrm{max}$) & 240\,s \\
Simulation timestep ($\Delta t$) & 0.5\,s \\
Waiting time ($d(e(v,v))$) & 5\,s \\
Robot maximum velocity & 1.2\,m/s \\
Pedestrian velocity & 1.0\,m/s \\
Transition prob.\ (next / skip / stay) & 0.75 / 0.05 / 0.2 \\
Prob.\ of entering each loop & 0.5 \\
\bottomrule
\end{tabular}
\end{table}

\end{document}